\documentclass{article} 

\usepackage{amsmath, amssymb, graphicx, subfigure,caption, fullpage, parskip, enumerate}
\usepackage{caption, calc}
\newcommand{\G}{\mathcal{G}}
\newcommand{\A}{\mathcal{A}}
\newcommand{\N}{\mathcal{N}}
\newcommand{\R}{\mathbb{R}}
\newcommand{\E}{\mathbb{E}}

\newtheorem{theorem}{Theorem}[section]
\newtheorem{lemma}[theorem]{Lemma}
\newtheorem{proposition}[theorem]{Proposition}
\newtheorem{corollary}[theorem]{Corollary}

\newenvironment{proof}[1][Proof]{\begin{trivlist}
\item[\hskip \labelsep {\bfseries #1}]}{\end{trivlist}}
\newenvironment{definition}[1][Definition]{\begin{trivlist}
\item[\hskip \labelsep {\bfseries #1}]}{\end{trivlist}}

\newcommand{\qed}{\nobreak \ifvmode \relax \else
      \ifdim\lastskip<1.5em \hskip-\lastskip
      \hskip1.5em plus0em minus0.5em \fi \nobreak
      \vrule height0.75em width0.5em depth0.25em\fi}

\title{Tight Measurement Bounds for Exact Recovery of Structured Sparse Signals}

\author{Nikhil Rao$^1 ~\ $ Benjamin Recht$^2 ~\ $  Robert Nowak$^1 ~\ $ \\
\\
$^1$ Department of Electrical and Computer Engineering \\
$^2$ Department of Computer Sciences \\
University of Wisconsin-Madison}

\date{}
\begin{document}
\maketitle

\begin{abstract}
Standard compressive sensing results state that to exactly recover an $s$ sparse signal in $\R^p$, one requires $\mathcal{O}(s \cdot \log p)$ measurements. While this bound is extremely useful in practice, often real world signals are not only sparse, but also exhibit structure in the sparsity pattern. We focus on group-structured patterns in this paper.  Under this model, groups of signal coefficients are active (or inactive) together.  The groups are predefined, but the particular set of groups that are active (i.e., in the signal support) must be learned from measurements. We show that exploiting  knowledge of groups can further reduce the number of measurements required for exact signal recovery, and derive universal bounds for the number of measurements needed. The bound is universal in the sense that it only depends on the number of groups under consideration, and not the particulars of the groups (e.g., compositions, sizes, extents, overlaps, etc.).  Experiments show that our result holds for a variety of overlapping group configurations.
\end{abstract}

\section{Introduction}

In many fields such as genetics, image processing, and machine learning, one is faced with the task of recovering very high dimensional signals from relatively few measurements. In general this is not possible, but fortunately many real world signals are, or can be transformed to be, sparse, meaning that only a small fraction signal coefficients are non-zero. Compressed Sensing \cite{CRT,donoho} allows us to recover sparse, high dimensional signals with very few measurements. In fact,  results indicate that one only needs $\mathcal{O}(s\cdot \log p)$ random measurements to exactly recover an $s$ sparse signal of length $p$.

In many applications however, one not only has knowledge about the sparsity of the signal, but some additional information about the structure of the sparsity pattern as well: 
\begin{itemize} 
\item In genetics, the genes are arranged into pathways, and genes belonging to the same pathway are highly correlated with each other \cite{pathway}.
\item In image processing, the wavelet transform coefficients can be modeled as belonging to a tree, with parent-child coefficients exhibiting similar properties \cite{crouse98, romberg, nricip11}.
\item In wideband spectrum sensing applications, the spectrum typically displays clusters of non-zero frequency coefficients, each corresponding to a narrowband transmission \cite{analogCS}
\end{itemize}

In cases such as these, the sparsity pattern can be represented as a union of certain groups of coefficients (e.g., coefficients in certain pathways, tree branches, or clusters).
This knowledge about the signal structure can help further reduce the number of measurements one needs to exactly recover the signal. Indeed, the authors in \cite{huang09} derive information theoretic bounds for the number of measurements needed for a variety of signal ensembles, including trees. In, \cite{modelbased,model09}, the authors show that one needs far fewer measurements when the signal can be expressed as lying in a union of subspaces, and explicit bounds are derived when using a modified version of CoSaMP \cite{cosamp} to recover the signal. In this paper, we derive bounds on the number of random iid gaussian measurements needed to exactly recover a sparse signal when its pattern of sparsity lies in a union of groups, when solving the \emph{convex} recovery algorithm introduced in \cite{jacob}.

We analyze the group-structured sparse recovery problem using a random Gaussian measurement model. We emphasize that although the derivation assumes the measurement matrix to be Gaussian, it can be extended to \emph{any} subgaussian case, by paying a small constant penalty, as shown in \cite{subgaussian}. We restrict ourselves to the Gaussian case here since it highlights the main ideas and keeps the analysis as simple as possible.

Note that in this work, variables can be grouped into arbitrary sets, and we make \emph{no} assumptions about the nature of the groups, except that they are known in advance. In short, we derive bounds for any generic group structure of variables, whether the groups overlap or form a partition of the ambient high dimensional space.


To the best of our knowledge, these results are new and distinct from prior theoretical characterizations of group lasso methods. Asymptotic consistency results are derived for the group lasso when the groups partition the space of variables in \cite{bachConsistency}. Similarly, in \cite{zhang09}, the authors consider the groups to partition the space, and derive conditions for recovery using the group lasso \cite{yuanlin}. In \cite{jenatton, jenatton10}, the authors derive consistency results for the group lasso under arbitrary groupings of variables. In \cite{Mest}, the authors consider overlapping groups and derive sample bounds under the group lasso \cite{yuanlin} setting. The authors in \cite{jacob}  derive consistency results in an asymptotic setting, for the group lasso with overlap, but do not provide exact recovery results. The general group lasso scenarios is different from what we consider, in that the group lasso yields vectors whose support can be expressed as a complement of a union of groups, while we consider cases where we require the support to lie in a union of groups, a distinction made in \cite{jacob}. Note that in the case of non-overlapping groups, the complement of a union of groups is a union of (a different set of) groups.  In this paper, we (a) derive sample complexity bounds in a compressive-sensing framework when the measurement matrix is \emph{i.i.d.} gaussian. (b) We focus on non-asymptotic sample bounds, and in a  case where the support is contained in a union of groups, and (c) make no assumptions about the nature of groups.  To derive our results, we appeal to the notion of restricted minimum singular values of an operator.

We bound number of measurements needed for exact recovery with two terms. The first term grows linearly in the total number of non-zero coefficients (with a small constant of proportionality). This is close to the bare minimum of one measurement per non-zero component. The second term only depends on the number of groups under consideration, and not the particulars of the groups (e.g., compositions, sizes, extents, etc.). In particular, the groups need not be disjoint. The degree to which groups overlap, remarkably, has no effect on our bounds. In this regard, our bounds can be termed to be \emph{universal}. This is somewhat surprising since overlapping groups are strongly coupled in the observations, tempting one to suppose that overlap may make recovery more challenging.

Our main result shows that for signals with support on $k$ of $M$ possible groups, exact recovery is possible from $(\sqrt{2\log(M-k)} + \sqrt{B})^2k + kB$ measurements using an overlapping group lasso algorithm, B being the maximum group size. Note that the bound depends on the sparsity $s$ of the signal via the $kB$ term, which is a loose upper bound for $s$ when the groups highly overlap. This arises as an artifact of the general approach we use to bound the number of measurements, and in specific cases, this can be made much tighter.

Our proof derives from the techniques developed in \cite{venkat}. The rest of this paper is organized as follows: in Section \ref{sec:theory}, we lay the groundwork for the main contribution of the paper, \emph{viz.} applying the techniques from \cite{venkat} tot he specific setting of group lasso with overlapping groups. We describe the theory and reasoning behind this approach. In Section \ref{sec:normalcone} we derive bounds on the number of random \emph{i.i.d.} gaussian measurements needed to be taken for exact recovery of group sparse signals. We further derive bounds for the number of measurements required for robust recovery of signals as well. Section \ref{sec:results} outlines the experiments we performed and the corresponding results obtained. We conclude our paper in Section \ref{sec:conc}.

\subsection{Notations}
\label{sec:notations}
\vspace{-2mm}
We first introduce notations that we will use for the rest of the paper. Consider a signal of length $p$, that is $s$ sparse. Note here that in case of multidimensional signals like images, we assume they are vectorized to have length $p$. The coefficients of the signal are grouped into sets $\{G_i\}_{i=1}^M$, such that $\forall i \in \{1,2,\cdots,M\}, G_i \subset \{1,2, \cdots ,p\}$. We denote the set of groups by $\G = \{G_i\}_{i=1..M}$, and $|\cdot|$ denotes the cardinality of a set. We let $x^\star$ be the (sparse) signal to be recovered, whose non zero coefficients lie in $k$ of the $M$ groups $\G^\star \subset \mathcal{G}$.  Formally, 
\[
\G^\star = \left\{ G_i \in \G^\star : \mathrm{supp}(x^\star) \cap G_i \neq 0 \right\}
\]
We assume $|\G^\star| = k \leq M = |\G|$. We let $\Phi_{n\times p}$ be a measurement matrix consisting of \emph{i.i.d.} gaussian entries of mean 0 and unit variance so that every column is a realization of an \emph{i.i.d.} gaussian length $n$ vector with  covariance $I$. For any vector $x \in \mathbb{R}^p$, we denote by $x_G$ a vector in $\R^p$ such that $(x_G)_i = x_i$ if $i \in G$, and 0 otherwise. We denote the observed vector by $y \in \R^n ~\ : y = \Phi x^\star$. The absence of a subscript following a norm $|| \cdot ||$ implies the $\ell2$ norm. The dual norm of $\|\cdot\|_p$ is denoted by $\|\cdot\|_p^*$. The convex hull of a set of points $S$ is denoted by $\mathrm{conv}(S)$.

\section{Preliminaries}
\label{sec:theory}
\vspace{-2mm}
In this Section, we will set up the problem that we wish to solve in this paper. We will argue as to why exact recovery of the signal corresponds to the minimization of the atomic norm of the signal, with the atoms obeying certain properties governed by the signal structure. 
\subsection{Atoms and the atomic set}

To begin with, let us formalize the notion of atoms and the atomic norm of a signal (or vector). We will restrict our attention to group-sparse signals in $\R^p$, though the same concepts can be extended to other spaces as well. We assume that $x \in \R^p$ can be decomposed as :
\[
x = \sum_{i = 1}^k c_i a_i, ~\ c_i \geq 0 
\]
The vectors $a_i$ are called \emph{atoms}, and form the basic building blocks of any signal, which can be represented as a conic combination of the atoms. We denote $\A = \{ a_i\}$ to be the \emph{atomic set}. Given a vector $x\in\R^p$ and an atomic set, we define the \emph{atomic norm} as
\begin{equation}
\label{anormdef}
||x||_{\A} = \inf \left\{ \sum_{a \in \A} c_a : x =  \sum_{a \in \A} c_aa, ~\  c_a \geq 0   ~\ \forall a \in \A \right\}
\end{equation}
The atomic decomposition of the signal has been known to be the simplest representation of the signal in some sense. Hence, to obtain a ``simple" representation of a vector, we look to minimize the atomic norm subject to constraints (equation (\ref{minAnorm})): 
\begin{equation}
\label{minAnorm}
\hat{x} = \underset{x \in \R^p}{\operatorname{argmin}} ||x||_\A  ~\ \textbf{s.t. } y = \Phi x
\end{equation}
Indeed, when the atoms are merely the canonical basis in $\R^p$, the atomic norm reduces to the standard $\ell_1$ norm, and minimization of the atomic norm yields the well known \emph{lasso} procedure \cite{tibshirani}. 

Assuming we are aware of the group structure $\G$, we now proceed to define the atomic set and the corresponding atomic norm for our framework: 
\begin{align*}
&\forall G \in \G, \text{ let } \\  &A_G = \{a^G \in \mathbb{R}^p : ||(a^G)_G||_2 = 1, (a^G)_{G^c} = 0 \} 
\end{align*}
\begin{equation}
\label{aset}
\A = \{ A_G \}_{G \in \G}
\end{equation}
We now show that the atomic norm of a vector $x \in \R^p$ under the atomic set defined in equation (\ref{aset}) is equivalent to the overlapping group lasso norm defined in \cite{jacob}, a special case of which is the standard group lasso norm \cite{yuanlin}. Thus, minimizing the atomic norm in this case is exactly the same as the group lasso with overlapping groups.
\vspace{-1mm}
\begin{lemma}
\label{lemogl}
Given any arbitrary set of groups $\G$, we have 
\[
||x||_{\A} = \Omega_{overlap}^\G (x)
\]
where $\Omega_{overlap}^\G (x)$ is the overlapping group lasso norm defined in \cite{jacob}.
\end{lemma}
\begin{proof}
In (\ref{anormdef}), we can substitute $v_G = c_G  a$, giving us $c_G = |c_G| \cdot ||a|| = || c_G a|| = \|v_G\|$. Hence, 
\begin{align*}
||x||_{\A} &= inf \left\{ \sum_{a \in \A} c_a : x =  \sum_{a \in \A} c_aa ~\  c_a \geq 0 ~\  \forall a \in \A \right\} \\
&= inf \left\{ \sum_{G \in \G} ||v_G|| ~\ : x = \sum_{G \in \G} v_G \right\} \\
&= \Omega_{overlap}^\G(x)  \qed
\end{align*}
\end{proof}

\begin{corollary}
\label{corrgl}
Under the atomic set defined in \ref{aset} , when $\G$ partitions $\R^p$, 
\[
||x||_{\A} = \sum_{G \in \G} ||x_G||
\]
\end{corollary}
\begin{proof}
$\Omega_{overlap}^\G = \sum_{G \in \G} ||x_G||$ in the non overlapping case. \qed
\end{proof}

Thus, (\ref{minAnorm}) yields:
\begin{equation}
\label{minAnormov}
\hat{x} = \underset{x \in \R^p}{\operatorname{argmin}} ~\ \Omega_{overlap}^\G(x)  ~\ \textbf{s.t. } y = \Phi x
\end{equation}
which can be solved using  \cite{jacob}.

Also note that we can directly compute the dual of the atomic norm from the set of atoms
\begin{equation}\label{eq:atomic-norm-dual}
	\|u\|_{\A}^* = \sup_{a\in \A} \langle a,u \rangle = \max_{G\in \G} ||u_G||
\end{equation}
The dual norm will be useful in our derivations below.

\subsection{Gaussian Widths and Exact Recovery}

Following \cite{venkat}, we define the \emph{tangent cone} and \emph{normal cone} at $x^\star$ with respect to $conv(\A)$ under $||x||_\A$ as \cite{wets}:
\begin{align}
\label{TconeD}
\mathcal{T}_{\A}(x^\star) &= \operatorname{cone} \{ z - x^\star ~\ : ||z||_\A \leq ||x^\star||_\A \}\\
\label{NconeD}
\mathcal{N}_\A(x^\star) &= \{ u~:~ \langle u,z \rangle \leq 0,~~\forall z\in \mathcal{T}_{\A}(x^\star)\} \\
\nonumber &= \{u~:~ \langle u, x^\star \rangle = t\|x\|_{\A}~ \\
\nonumber ~&\mbox{and}~\|u\|_{\A}^* \leq t~\mbox{for some}~t\geq 0\}
\end{align}

We note that, from \cite{venkat} (Prop. 2.1), $\hat{x} = x^\star$ (\ref{minAnorm}) is unique \emph{iff} 
\begin{equation}
\label{prop}
null(\Phi) \cap \mathcal{T}_{\A}(x^\star) = \{0\} 
\end{equation}
Hence, we require that the tangent cone at $x^\star$ intersects the nullspace of $\Phi$ only at the origin, to guarantee exact recovery. 

Before we state the main recovery result from~\cite{venkat}, we define the \emph{gaussian width} of a set:
\begin{definition}
Let $\mathbb{S}^{p-1}$ denote the unit sphere in $\R^p$.  The Gaussian width $\omega(S)$ of a set $S \in \mathbb{S}^{p-1}$ is 
\[
\omega(S) = \mathbb{E}_g \left[ \sup_{z \in S} g^Tz \right]
\]
where $g \sim \mathcal{N}(0,I)$
\end{definition}
Gordon uses the Gaussian width to provide bounds on the probability that a random subspace of a certain dimension misses a subset of the sphere~\cite{gordon}.  In~\cite{venkat}, these results are specialized to the case of atomic norm recovery.  In particular, we will make use of the following: 

\begin{proposition} \label{corl:width}[\cite{venkat}, Corollary 3.2]
Let $\Phi: \R^p \rightarrow \R^n$ be a random map with i.i.d. zero-mean Gaussian entries having variance $1/n$.  Further let $\Omega = T_\A(x^*) \cap \mathbb{S}^{p-1}$ denote the spherical part of the tangent cone $T_\A(x^\star)$.  Suppose that we have measurements $y = \Phi x^\star$, and we solve the convex program (\ref{minAnorm}). Then $x^\star$ is the unique optimum of (\ref{minAnorm}) with high probability provided that
	\begin{equation*}
	n \geq \omega(\Omega)^2+\mathcal{O}(1).
	\end{equation*}
\end{proposition}

To complete our problem setup we will also restate Proposition 3.6 in \cite{venkat} :
\begin{proposition} (Proposition 3.6 in \cite{venkat})
\label{prop3.6}
Let $C$ be any non-empty convex cone in $\R^p$, and let $g \sim \N(0,I)$ be a Gaussian vector. Then:
\begin{equation}
\label{propVenkat}
\omega(C \cap \mathbb{S}^{p-1}) \leq \E_g[\mathrm{dist}(g,C^*)]
\end{equation}
where $\mathrm{dist}( . , .)$ denotes the Euclidean distance between a point and a set, and $C^*$ is the dual cone of $C$ 
\end{proposition}

We can then square (\ref{propVenkat}) use Jensen's inequality to obtain
\begin{equation}
\label{jensen}
\omega(C \cap \mathbb{S}^{p-1})^{2} \leq \E_g[\mbox{dist}(g,C^*)^2]
\end{equation}
We note here that the dual cone of the Tangent cone is the Normal cone, and vice-versa.

Thus, to derive measurement bounds, we only need to calculate the square of the gaussian width of the interSection of the tangent cone at $x^\star$ with respect to the atomic norm and the unit sphere. This value can be bounded by the distance of a gaussian random vector to the Normal cone at the same point, as implied by  (\ref{jensen}). In the next Section, we derive bounds on this quantity.

\section{Gaussian Width of the Normal Cone of the Group Sparsity Norm}
\label{sec:normalcone}

For generic groups $\G$, we have

\begin{align}
\label{eqn:normalGen}
v \in \N_\A(x^\star) \Leftrightarrow 
\exists \gamma \geq 0 ~\ :  \langle v, x^\star \rangle = \gamma \|x^\star\|_{\A}, 
 \|v_G\|=\gamma \mbox{ if } G\in \G^\star ,~\  \|v_G\|\leq \gamma \text{ if } G\not\in \G^\star.
\end{align}
It is not hard to see that, in the case of disjoint groups, 
\begin{align}
\label{disjointcone}
\mathcal{N}_\A(x^\star) &= \{ z \in \mathbb{R}^p : z_i = \gamma \frac{(x^\star)_i}{||x^\star||}  ~\ \forall G \in \G^\star , \\  \notag &~\   ||z_G|| \leq \gamma ~\ \forall G ~\notin ~\G^\star , \gamma  \geq 0 \} \end{align}
However, in the case of overlapping groups, we do not know how to obtain such a closed from.  

We now prove the main result of this paper, a sufficient number of gaussian measurements needed to recover a group-sparse signal:

\begin{theorem}
\label{mainTh}
To exactly recover a $k$- group sparse signal decomposed into $M$ groups in $\R^p$, $(\sqrt{2\log(M-k)} + \sqrt{B})^2 k + kB$ iid gaussian measurements are sufficient.
\end{theorem}

To prove this result, we need two lemmas:

\begin{lemma}\label{lemma:chisq}
Let $q_1,\ldots,q_L$ be $L$, $\chi$-squared random variables with $d$-degrees of freedom.  Then \[
\E[\max_{1\leq i \leq L} q_i]\leq (\sqrt{2\log(L)} + \sqrt{d})^2.
\]
\end{lemma}



\begin{proof}
Let $M_L := \max_{1\leq i \leq L} q_i$. For $t>0$, we have that
\begin{align*}
\E[M_L] &= \frac{\log[\exp(t\cdot \E[M_L])]}{t} \\
               &\stackrel{(i)}\leq \frac{\log[\E[\exp(t \cdot M_L)]]}{t} \\
               &\stackrel{(ii)}= \frac{\log[\E [\max_{1\leq j \leq L} \exp(t \cdot q_j)]]}{t} \\
               &\stackrel{(iii)} \leq \frac{\log[L \E[\exp(t \cdot q_1)]]}{t} \\
               &= \frac{\log(L) - \frac d2 \log(1 - 2t)}{t} 
\end{align*}
Where (i) follows from Jensen's inequality , (ii) follows from the monotonicity of the exponential function, and (iii) merely bounds the maximum by the sum over all the elements. Now, setting $t = (2 + 2\epsilon)^{-1}$ with $\epsilon = \sqrt{\frac{d}{2\log(L)}}$ yeilds $\E[M_L] \leq (\sqrt{2\log(L)} + \sqrt{d})^2$ \qed
\end{proof}

Note that $t$ can be optimized depending on the application.    We use this particular choice because it makes no assumptions about the relative magnitudes of $(M-k)$ and $B$.

\begin{lemma}\label{lemma:ballbound}
Suppose $v\in \R^p$ is supported on some set of groups $\G^\star \subset \G$
\[
	\|v\| \leq \sqrt{|\G^\star|}~\ \|v\|_{\A}^*\,.
\]
\end{lemma}
\begin{proof}
	By duality, it suffices to show that $\|z\|_{\A} \leq \sqrt{|\G^\star|}~\ \|z\|$ for all $z$.  For any $z$, there exists a representation $z = \sum_{G\in \G^\star} b_G$ where none of the supports of $b_G$ overlap.  It then follows that
	\begin{align*}
	\|z\|_{\A} &\stackrel{(i)}\leq \sum_{G\in \G^\star} \|b_G\| \\ &\stackrel{(ii)}\leq \sqrt{|\G^\star|} \left(\sum_{G\in \G^\star} \|b_G\|^2\right)^{1/2} \\  &= \sqrt{|\G^\star|}~\ \|z\|
	\end{align*}
	Where (i) follows from the definition of the norm $\|\cdot\|_\A$ and (ii) is a consequence of the relation $\|\beta\|_1 \leq \sqrt{k} \| \beta \|_2$ for $k$ dimensional vectors $\beta$ \qed
\end{proof}

\begin{proof}[Proof of Theorem~\ref{mainTh}]

\emph{Intuition}:
Note that, from (\ref{jensen}), we need to bound the distance between $\N_\A(x^\star)$ and a random gaussian vector. In the following proof, we carefully construct a specific vector $r \in \N_\A(x^\star)$ and bound the distance from $r$ to the gaussian vector. Naturally, this will be an upper bound to the distance desired.

Now, let $S  =\cup_{G \in \G^\star} G$, \emph{i.e.} $S$ is the indices corresponding to the union of groups that support $x^\star$. Note that $S \subset \{1,2,\ldots,p\}$.
Since the normal cone is nonempty, let $v \in \N_\A(x^\star)$ and let
\begin{equation}
\label{eq:dualunity}
\|v\|^*_\A = 1
\end{equation}
then we must have that $\langle v,x^\star \rangle = \|x^\star\|_\A$. Moreover, for each $G$ that intersects $S$, $\|v_G\| = 1$. This follows from the definition in (\ref{eqn:normalGen}). Also, let $v_{S^c} = 0$.  It can be verified that $v$ satisfies all the properties in (\ref{eqn:normalGen}).



Let $w \sim \N(0,I_p)$ be a vector with iid gaussian entries. We can write $w = [w_S ~\ w_{S^c}]^T$. Let $t(w) = \max_{G\not\in \G^\star} \|w_G\|$. 

Let us now construct a vector $r\in \N_\A(x^\star)$. We can decompose $r$ as $r = [r_S ~\ r_{S^c}]^T$. Let $r_S = t(w) \cdot v_S$, and $r_{S^c} = w_{S^c}$. 

From  \eqref{eqn:normalGen}, and from our definition of $t(w)$, we have $r \in \N_\A(x^\star)$. Referring to  (\ref{jensen}), we now consider the expected squared distance between $\N_\A(x^\star)$ and $w$:
\begin{align*}
\E[\mbox{dist}(w,C^*)] &\leq \E[|| r - w ||^2] \\
&= \E [\|r_S - w_S\|^2 + \|r_{S^c} - w_{S^c}\|^2] \\
&\stackrel{(i)}= \E [\|r_S - w_S\|^2] \\
&\stackrel{(ii)}= \E[ \|r_S\|^2 ] + \E [\|w_S\|^2]\\
&= \E[\|t(w) \cdot v_S\|^2] + \E[\|w_S\|^2]\\
&\stackrel{(iii)}=  \E[t(w)^2] \cdot \|v_S\|^2 + \E[\|w_S\|^2] \\
&\stackrel{(iv)}= \E[t(w)^2] \cdot \|v_S\|^2 + |S| \\
&\stackrel{(v)}\leq (\sqrt{2\log(M-k)} + \sqrt{B})^2 \cdot  \|v_S\|^2 + kB \\
&\stackrel{(vi)}\leq (\sqrt{2\log(M-k)} + \sqrt{B})^2 \cdot k + kB
\end{align*}

Where (i) follows because $S$ and $S^c$ are disjoint, (ii) follows from the fact that $r_S$ and $w_S$ are independent, (iii) follows from the fact that $v$ is deterministic. We obtain (iv) since $\|w_S\|^2$ is a $\chi^2$ random variable with $|S|$ degrees of freedom.  (v) follows from Lemma \ref{lemma:chisq}, and from the fact that $kB$ is a upper bound on the signal sparsity. Finally, (vi) follows from Lemma \ref{lemma:ballbound}, noting that ${| \G^\star |} \leq k$, and $\|v\|_\A^* = 1$ from (\ref{eq:dualunity}).\qed
\end{proof}

If the groups are disjoint to begin with , the the normal cone will be given by (\ref{disjointcone}), and $\|v_S\|^2 = k$.  Also, in this case, we have $|S| = kB$. We see that  we do not pay an additional penalty when the groups overlap, except that the bound for the signal sparsity becomes loose. This fact is surprising, since one would expect that one would need more measurements to effectively capture the dependencies among the overlapping groups. 

\subsection{Remarks}
The $k B$ term in the bound is an upper-bound on the signal sparsity. In the case of highly overlapping groups, this value may be much larger than the signal sparsity. This is an unfortunate artifact of the general approach we take to derive a bound on the number of measurements. If the specific structure of groups is known (trees, hierarchies, etc.), one can refine the bound accordingly. Of course, the bound will be tightest when there is a block-sparse structure, i.e. there is no overlap between groups. 

It can be seen from Theorem \ref{mainTh} that the number of measurements is linear in $k$ and $B$. Hence, the number of measurements that are sufficient for signal recovery grows linearly with the number of active groups in the signal, and also the maximum group size. 

We note that although we pay no extra price to measure the signal when the groups overlap, there is an additional cost in the recovery process of the signal, in that the groups need to first be separated by replication of the coefficients \cite{jacob}, or resort to a primal-dual method to solve the problem \cite{mosci}.

\subsection{Noisy Observations}
The results we obtain can be easily extended to the case where we obtain noisy observations, assuming that the noise is bounded. In the noisy case, we observe 
\[
y = \Phi x^\star + \theta , ~\  \|\theta\| \leq \delta
\]
We then solve the atomic norm minimization problem, with a relaxed constraint to take into account the bounded noise:
\begin{equation}
\label{minAnormNoise}
\hat{x} = \underset{x \in \R^p}{\operatorname{argmin}} ||x||_\A  ~\ \textbf{s.t. } \|y - \Phi x\| \leq \delta
\end{equation}
We restate corollary $3.3$ from \cite{venkat}:

\begin{proposition} \label{corl:widthnoisy}[\cite{venkat}, Corollary 3.3]
Let $\Phi: \R^p \rightarrow \R^n$ be a random map with i.i.d. zero-mean Gaussian entries having variance $1/n$.  Further let $\Omega = T_\A(x^*) \cap \mathbb{S}^{p-1}$ denote the spherical part of the tangent cone $T_\A(x^\star)$.  Suppose that we have measurements $y = \Phi x^\star + \theta$, and $\|\theta\| \leq \delta$. Suppose we solve the convex program (\ref{minAnormNoise}). Let $\hat{x}$ denote the optimum of (\ref{minAnormNoise}). Also, suppose $\| \Phi z \| \geq \epsilon \|z\| ~\ \forall z \in T_\A(x^\star)$. Then $\|x^\star - \hat{x}\| \leq \frac{2\delta}{\epsilon}$  with high probability provided that
	\begin{equation*}
	n \geq \frac{\omega(\Omega)^2}{(1-\epsilon)^2}+\mathcal{O}(1).
	\end{equation*}
\end{proposition}

Substituting the result in Theorem \ref{mainTh} in Proposition \ref{corl:widthnoisy}, we have the following corollary yielding a sufficient condition to accurately recover a signal when the measurements are corrupted with bounded noise:

\begin{corollary}
\label{corr:noisy}
Suppose we wish to recover a $k-$ group sparse signal having $M$ groups, such that the maximum group size is $B$. Let $\hat{x}$ be the optimum of the convex program (\ref{minAnormNoise}). To have $\|\hat{x} - x^\star\| \leq \frac{2\delta}{\epsilon}$ with high probability, 
\[
 \frac{(\sqrt{2\log(M-k)} + \sqrt{B})^2 k + kB}{(1-\epsilon)^2}
 \]
 iid Gaussian measurements are sufficient.
 \end{corollary}

\section{Experiments and Results}
\label{sec:results}
\vspace{-3mm}
We extensively tested our method against the standard lasso procedure. In the case where the groups overlap, we use the replication method outlined in \cite{jacob}, to reduce the problem to that of non overlapping groups. \vspace{-1mm}
We compare the number of measurements needed for our method with that needed for the lasso. For the lasso, we use the bound derived in \cite{venkat} , \emph{viz.} $(2s+1)\log(p-s)$.  We generate length $p =2000$ signals, made up of $M = 100$ non-overlapping groups of size $B = 20$. We set $k = 5$ groups to be ``active", and the values within the groups are drawn from a uniform $[0,1]$ distribution. The active groups are assigned uniformly at random. The sparsity of the signal will be $s = 100$

We use SpaRSA \cite{sparsa} for the lasso and the group lasso with overlap, learning $\lambda$ over a grid. Figure \ref{fig:complasso} displays the mean reconstruction error $||\hat{x} - x^*||_2^2 / p$ as a function of the number of random measurements taken. The errors have been averaged over 100 tests, and each time a new random signal was generated with the above mentioned parameters. 

From the parameters considered, we conclude that $\approx 380$ measurements are sufficient to recover the signal. Note that, when we have 380 measurements, the lasso does not recover the signal exactly, as seen in Figure \ref{fig:complasso}.

\begin{figure}[!h]
\centering
\includegraphics[scale = 0.38]{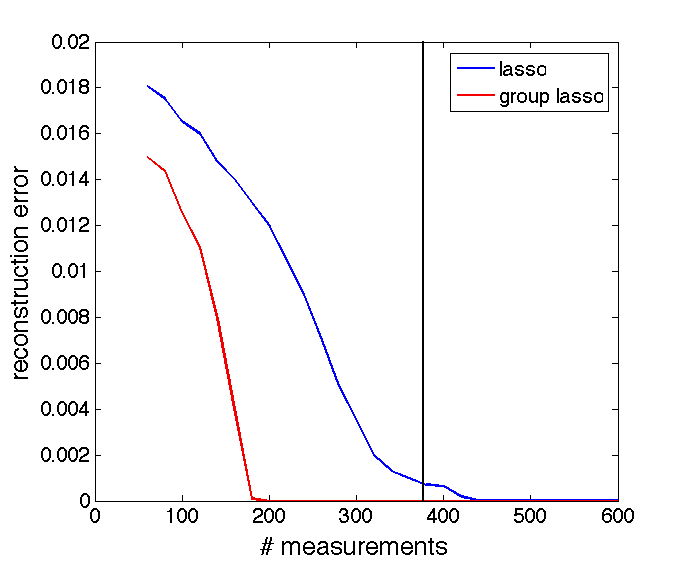}
\caption{Comparison with the lasso. The vertical line indicates our bound. Note that our bound (380) predicts exact recovery of the signal, while at the same value, the lasso does not recover the signal}
\label{fig:complasso}
\end{figure}

To show that the bound we compute holds regardless of the complexity of groupings, we consider the following scenario: Suppose we have $M = 100$ groups, each of size $B = 40$. $k = 5$ of those groups are active, and the values within each group are assigned from a uniform $[-1,1]$ distribution. We arrange these groups in three configurations:

\begin{enumerate}[(i)]
\item The groups do not overlap, yielding a signal of length $p = 4000$, and signal sparsity $s = 200$.
\item A partial overlapping scenario, where apart from the first and last group, every group has $20$ elements in common with a group above it, and $20$ common with the group below, giving $p = 2020$, $s \in [120, ~\ 200]$ depending on which of the 100 groups are active.
\item An almost complete overlap, where apart from one element in each group, the remaining elements are common to each group. This leads to $p = 139$ and $s = 44$
\item We also considered cases intermediate to the ones listed above. Specifically, we  considered (a) a highly overlapping scenario which is identical to the previous case, but with odd and even groups disjoint. We also consider (b) a random overlap case where the first 50 groups are non overlapping and the remaining 50 are assigned uniformly at random from the existing $p = 2000$ indices.
\end{enumerate}
The scenarios we consider are depicted in Figure \ref{fig:scenes}. In each of the cases, we compute the bound to be $\approx 630$. We can see from Figure \ref{fig:scenes_results} that the bound holds for all cases. The bound becomes looser as the complexity of the groupings increases. This, as argued before, is a result of the bound for the signal sparsity becoming looser. From the values of $p$ and $s$ computed for the three cases, we have the corresponding bounds for the lasso to be 3305 for the no overlap case (i), [1819, 3010] for the partial overlap case (ii) and 405 for the almost complete overlap case (iii)respectively. The lasso bounds for case (iv) will lie between those for case (ii) and (iii).

\begin{figure}[!h]
\begin{center}
\subfigure[types of groupings considered. Each set of coefficients encompassed by one color belongs to one group]{
\includegraphics[scale = 0.7]{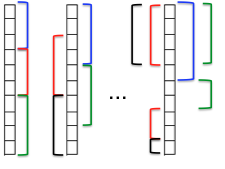}
\label{fig:scenes}}
\subfigure[performance on cases considered in figure \ref{fig:scenes}]{
\includegraphics[scale = 0.6]{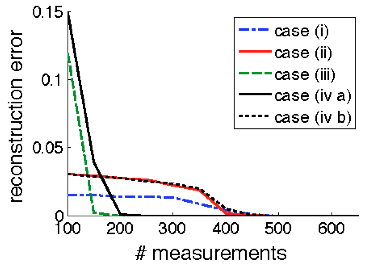}
\label{fig:scenes_results}}
\caption{(Best seen in color) Performance on various grouping schemes. Note that our bound evaluates to 630, clearly sufficient measurements to recover the signal. The corresponding bounds for the lasso (for cases (i), (ii) and (iii)) are 3305, [1819, 3010] (depending on $s$) and 405  respectively. We can see that, as the amount of overlap increases, our bound loosens, and for pathological cases the lasso bound is tighter}
\end{center}
\end{figure}

Finally, we consider the wavelet transform coefficients of the ``blocks" signal (Figure \ref{fig:blocks}). It was shown in \cite{nricip11} that the coefficients can be grouped, to account for parent child dependencies across scales of the wavelet transform, as shown in Figure \ref{fig:PC}.  In this case, for a $p = 16384$ length signal, we have $M = 16382$ groups, and the maximum group size is $B = 2$. We use the Haar wavelet bases to decompose the image. Figure \ref{fig:recons} shows the reconstruction obtained from $1690$ measurements, corresponding to the bound computed for $k = 47$.  We see that our bound yields a sufficient number of measurements for exact recovery.

\begin{figure}[!h]
\begin{center}
\includegraphics[scale = 0.5]{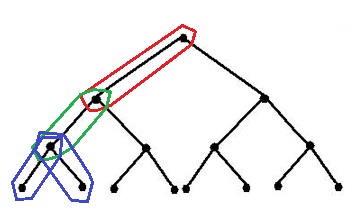}
\caption{Groups on the 1d wavelet transform}
\label{fig:PC}
\end{center}
\end{figure}

\begin{figure}[!h]
\begin{center}
\subfigure[original signal]{
\includegraphics[scale = 0.45]{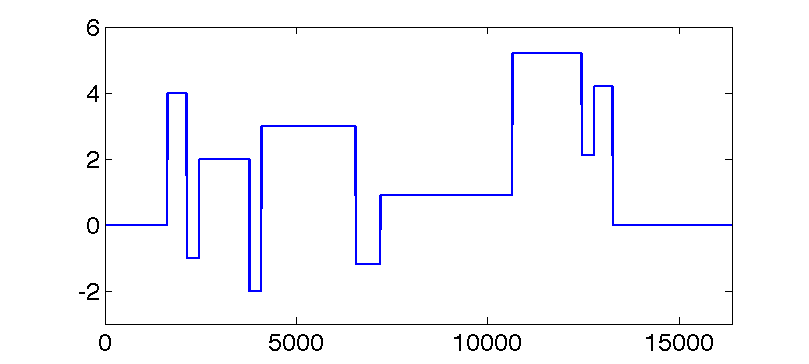}
\label{fig:blocks}}
\subfigure[reconstruction]{
\includegraphics[scale = 0.45]{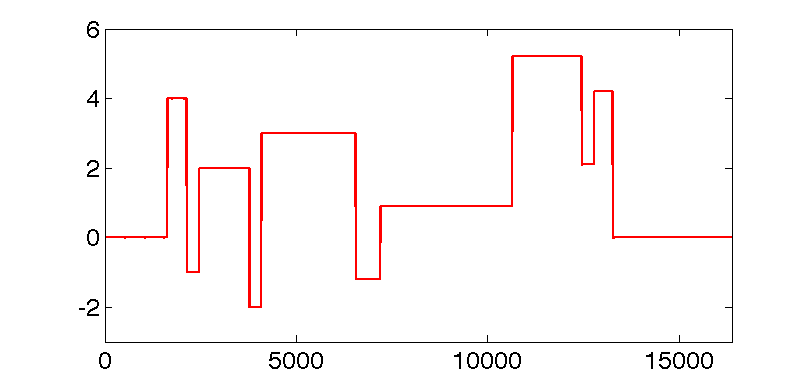}
\label{fig:recons}}
\caption{Exact reconstruction of a length 16384 signal from 1690 measurements in the wavelet domain}
\end{center}
\end{figure}

\vspace{-2mm}
\section{Conclusion}
\label{sec:conc}
\vspace{-2mm}
We showed that, when additional structure about the support of the signal to be estimated is known, we can recover the signal in much fewer measurements that what would be needed in the standard compressed sensing framework. Also, we showed that we surprisingly do not pay an extra penalty when the groups overlap each other. Moreover, the bound holds for arbitrary group structures, and can be used in a variety of  applications. The bounds we derive are tight, and can be extended to subgaussian measurement matrices by incurring a constant penalty. Experimental results on both toy and real data agree with the bounds we obtained.

\subsubsection*{Acknowledgements}
The authors wish to thank Waheed Bajwa and Guillaume Obozinski for insightful comments on the paper, which prompted several revisions to ensure correctness. 

\bibliographystyle{plain}
\bibliography{GS_measurement_arxiv}

\end{document}